\documentclass[lettersize,journal]{IEEEtran}
\usepackage{times}
\usepackage[numbers,sort]{natbib}
\usepackage{multicol}
\usepackage[bookmarks=true]{hyperref}\usepackage{amssymb,amsmath,amsthm, amsfonts}
\usepackage{xspace}
\usepackage{bm}
\usepackage{calc}
\usepackage{multirow}

\usepackage{tabularx, booktabs} % 
\newcolumntype{F}{>{\centering\arraybackslash}p{1.9cm}}

\newcolumntype{G}{>{\centering\arraybackslash}p{1.2cm}}
\newcolumntype{H}{>{\centering\arraybackslash}p{0.4cm}}
\newcolumntype{I}{>{\centering\arraybackslash}p{0.92cm}}
\newcolumntype{Z}{>{\centering\arraybackslash}p{-0.92cm}}

\usepackage{wrapfig}

\usepackage{graphicx}
\usepackage[font={small,it}]{caption}

\usepackage{algorithm}
\usepackage[noend]{algpseudocode}
\usepackage{algpseudocode} % 

\let\labelindent\relax
\usepackage{enumitem}

\newtheorem{theorem}{Theorem}
\newtheorem{lemma}{Lemma}
\newtheorem{corollary}{Corollary}
\newtheorem{remark}{Remark}
\newtheorem{example}{Example}

\def\namedlabel#1#2{\begingroup
    #2% 
    \def\@currentlabel{#2}% 
    \phantomsection\label{#1}\endgroup
}

\newcommand\sde{(\ref{eq:SDE})\xspace}
\newcommand\A[1]{\hyperref[A#1]{(A#1)}\xspace}
\newcommand\Bassum[1]{\hyperref[B#1]{(B#1)}\xspace}

\newcommand\AverageSmallMatrix[1]{{% 
  \small\arraycolsep=1\arraycolsep\ensuremath{\begin{bmatrix}#1\end{bmatrix}}}}

\usepackage{color}
\usepackage{xcolor}
\newcommand\rev[1]{#1}

\def\Inf{\operatornamewithlimits{inf\vphantom{p}}}

\newcommand{\Prob}{\mathbb{P}}
\newcommand{\bProb}{\bar{\mathbb{P}}}
\newcommand{\dd}{\textrm{d}} 
\newcommand{\dt}{\textrm{d}t} 
\newcommand{\bomega}{\bar\omega}

\newcommand{\G}{\mathcal{G}}
\newcommand{\F}{\mathcal{F}}

\newcommand{\Obs}{\mathcal{O}}

\newcommand{\bD}{\mathbb{D}}

\newcommand{\U}{\mathcal{U}}

\newcommand{\E}{\mathbb{E}}
\newcommand{\R}{\mathbb{R}}
\newcommand{\N}{\mathbb{N}}

\newcommand{\ocp}{\textbf{OCP}\xspace}
\newcommand{\socp}{\textbf{SOCP}\xspace}
\newcommand{\socphat}{\widehat{\socp}\xspace}

\newcommand{\cvar}{\textrm{CV@R}\xspace}
\newcommand{\avar}{\textrm{AV@R}\xspace}
\newcommand{\var}{\textrm{V@R}\xspace}
\newcommand{\ipopt}{\textrm{IPOPT}\xspace}

\newcommand{\osqp}{\textrm{OSQP}\xspace}

\begin{document}

\title{Risk-Averse Trajectory Optimization\\via Sample Average Approximation}
\author{\authorblockN{Thomas Lew$^{1}$, \  Riccardo Bonalli$^{2}$, \ Marco Pavone$^{1}$
\thanks{This work was supported by the NASA University Leadership Initiative (grant\#80NSSC20M0163) and the Air Force under an STTR award with Altius Space Machines, but solely reflects the opinions and conclusions of its authors. Toyota Research Institute provided funds to support this work.}}
%\thanks{Manuscript submitted on July 6, 2023.}}
\vspace{1mm}
\\
\authorblockA{\scalebox{0.9}{$^{1}$Department of Aeronautics and Astronautics, Stanford University}}
\authorblockA{\scalebox{0.9}{$^{2}$Laboratory of Signals and Systems, University of Paris-Saclay, CNRS, CentraleSup\'elec}}
\vspace{-6mm}
}

\maketitle

\begin{abstract}
Trajectory optimization under uncertainty underpins a wide range of applications in robotics. 
However, existing methods are limited in terms of reasoning about sources of epistemic and aleatoric uncertainty, space and time correlations, nonlinear dynamics, and non-convex constraints. 
In this work, % 
we first introduce a continuous-time planning formulation with an average-value-at-risk constraint over the entire planning horizon. % 
Then, we propose a sample-based approximation that unlocks an efficient \rev{and} general-purpose % 
algorithm for risk-averse trajectory optimization. 
We prove that the method is asymptotically optimal and derive finite-sample error bounds. 
Simulations demonstrate the high speed and reliability of the  approach on problems with stochasticity in nonlinear 
dynamics, obstacle fields, interactions, and terrain parameters.
\end{abstract}

\IEEEpeerreviewmaketitle

\section{Introduction}\label{sec:introduction}
Accounting for uncertainty in the design of decision-making systems 
is key to achieving reliable robotics autonomy. 
Indeed, modern autonomy stacks account for uncertainty \cite{Agha2021,Tranzatto2022}, whether it comes from 
noisy sensor measurements (e.g., due to perceptually-degraded conditions or a lack of features \cite{Falanga2018,Agha2021}), 
dynamics (e.g., due to disturbances and difficult-to-characterize nonlinearities \cite{deisenroth2015,hewing2018cautious,LewEtAl2022}), 
properties of the environment (e.g., due to unknown terrain properties for legged robots \cite{Drnach2021,Fan2021} and Mars rovers \cite{Higa2019,Cunningham2017}), or 
interactions with other agents (e.g., in autonomous driving \cite{SalzmannIvanovicEtAl2020,IvanovicElhafsiEtAl2020,Ren2023}). 

Although trajectory optimization under uncertainty underpins a wide range of applications, 
existing approaches % 
often make simplifying assumptions and approximations that reduce the range of problems they can reliably deal with. Specifically, there is a lack of methods capable of simultaneously handling
\begin{itemize}[leftmargin=3mm]
\item sources of aleatoric uncertainty (e.g., external disturbances) and epistemic uncertainty (e.g., a drone transporting a payload of uncertain mass that introduces time correlations over the state trajectory) 
    that depend on state and control variables (e.g., interactions between different agents),
\item uncertainty of arbitrary (non-Gaussian) distribution correlated over time and space (e.g., uncertain terrain properties), 
\item uncertain nonlinear dynamics and non-convex constraints, 
\item \rev{trajectory-wise} constraints that % 
bound the risk of constraints violations over the entire duration of the problem\rev{.} % 
\end{itemize}
\rev{As distinct systems may need specific numerical schemes, such methods should also have discretization-independent guarantees.} 
Table \ref{table:approaches} summarizes the capabilities of existing methods.

\newpage

\begin{figure}[!t]
\centering
\includegraphics[width=1\linewidth]{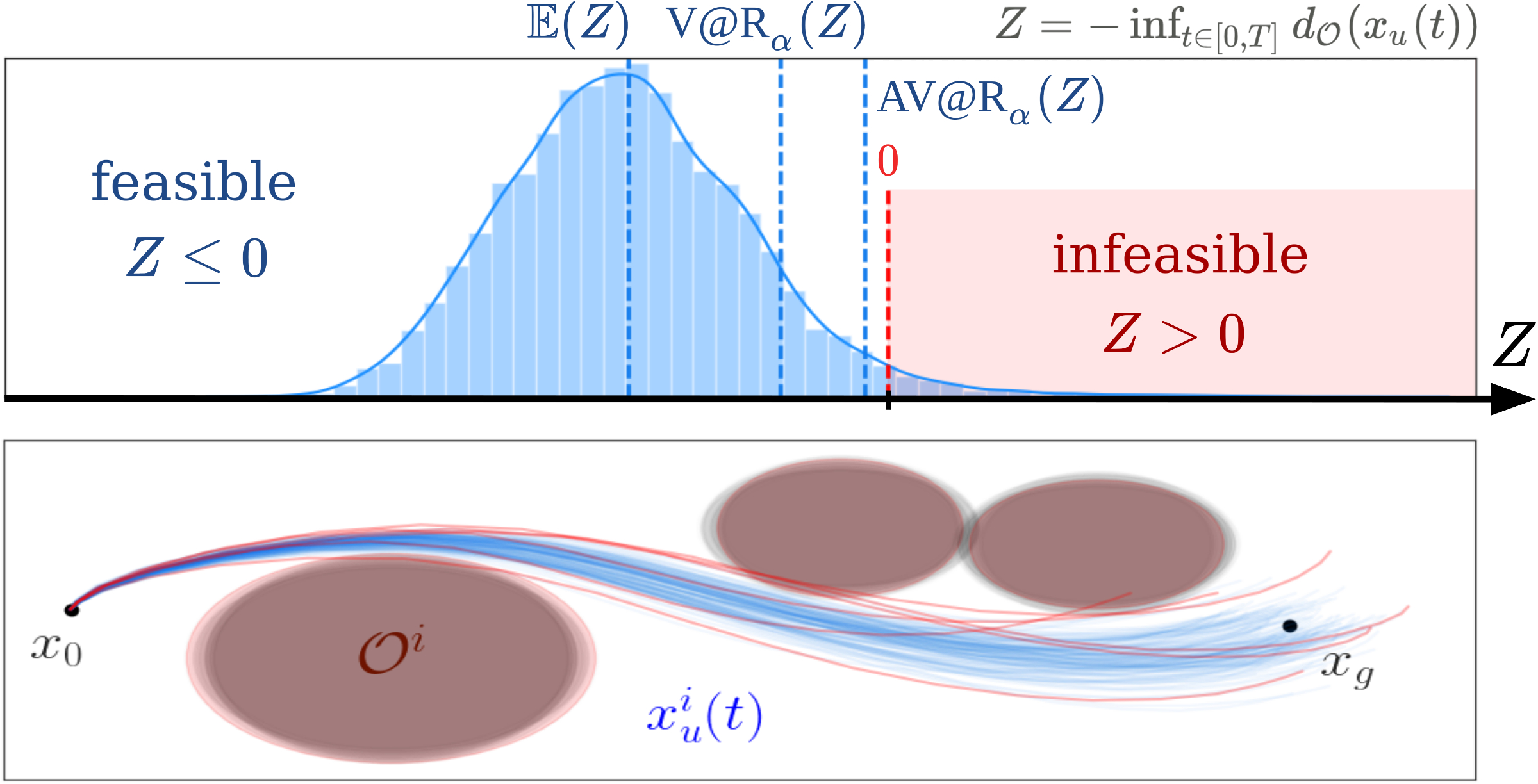}
\caption{Full-horizon collision statistics. We consider problems with average-value-at-risk (\avar) constraints enforced over the entire trajectory of the system. Uncertain obstacles and nonlinear dynamics, with uncertain inertia properties and  disturbances modeled as Brownian motion, introduce sources of aleatoric and epistemic uncertainty that are notoriously challenging to handle with existing approaches. % 
\vspace{-6mm}
}% 
\label{fig:main_figure_drone}
\end{figure}

\textbf{Contributions:} 
We introduce an efficient  risk-averse trajectory optimization algorithm satisfying the previous desiderata:
\begin{itemize}[leftmargin=3mm]
    \item First, we propose a risk-averse planning formulation with average-value-at-risk (\avar\cite{Shapiro2014}) constraints enforced over the entire planning horizon. 
    This formulation is applicable to a wide range of robotics problems with sources of aleatoric and epistemic uncertainty. 
	\rev{Its continuous-time nature % 
	guides the design of algorithms whose properties are independent of the chosen time discretization scheme (see Remark \ref{remark:pointwise_cc}). }% 
    Enforcing \avar constraints enables accounting for tail events and facilitates numerical resolution due to their convexity properties (Remark \ref{remark:avar_easy_optimize_vs_ccs}). \rev{In addition}, solving this formulation \rev{gives} feasible solutions to notoriously challenging problems with joint chance constraints (Remark \ref{remark:ccs_vs_avar}).
    \item Second, we propose a sample-based approximation rooted in the sample average approximation approach \cite{Pagnoncelli2009,Shapiro2014,LewBonalliEtAl2022}. 
    We derive asymptotic optimality guarantees  (Theorem \ref{thm:optimality}) and finite-sample error bounds (Lemma \ref{lem:finite_sample_socp}) for this reformulation. 
    The analysis relies on mild assumptions (\A{1}-\A{4}), which enables considering a wide range of uncertainty sources that introduce spatial and temporal correlations and depend on state and control variables. % 
    The resulting approximated problem is smooth and sparse, facilitating efficient numerical resolution using off-the-shelf optimization tools. 
\item Finally, we show the reliability and speed of the proposed approach on problems with uncertain nonlinear 
dynamics, obstacle fields, interactions, and terrain parameters.
\end{itemize}

\begingroup
\def\arraystretch{1.}
\setlength\tabcolsep{1.75mm}
\begin{table*}[!t]
\caption{Capabilities of existing approaches for risk-constrained planning and trajectory optimization.}
\label{table:approaches}
\centering
\begin{tabular}{c|c|c|c|c||c|c|c|c|}
\cline{2-9}
& 
	\begin{tabular}{c} 
	dynamics\\uncertainty
	\end{tabular}
& 
	\begin{tabular}{c} 
	constraints\\uncertainty
	\end{tabular}
& 
	\begin{tabular}{c} 
	aleatoric\\uncertainty
	\end{tabular}
& 
	\begin{tabular}{c} 
	epistemic\\uncertainty
	\end{tabular}
& 
	\begin{tabular}{c} 
	nonlinear\\dynamics
	\end{tabular}
& 
	\begin{tabular}{c} 
	non-convex\\constraints
	\end{tabular}
& 
	\begin{tabular}{c} 
	non-Gaussian \\ 
	distributions
	\end{tabular}
& 
	\begin{tabular}{c} 
	time \& space \\ 
	correlations
	\end{tabular}
\\
\cline{1-9}
\multicolumn{1}{ |c| }{\cite{Hakobyan2019}} & 
 & $\bullet$ &  & $\bullet$
	& $\bullet$ &  & % 
 $\bullet$ & $\bullet$
\\
\cline{1-9}
\multicolumn{1}{ |c| }{\cite{JansonSchmerlingEtAl2015b}} & 
 $\bullet$  &   &  $\bullet$ & 
	 &  & $\bullet$ & % 
  $\bullet$ & $\bullet$
\\
\cline{1-9}
\multicolumn{1}{ |c| }{\cite{Safaoui2021}} &
	$\bullet$ & & $\bullet$ &
	 & 
	$\bullet$ & $\bullet$ & % 
 $\bullet$ &
\\
\cline{1-9}
\multicolumn{1}{ |c| }{\cite{Thomas2022}} &
	$\bullet$ & $\bullet$ & $\bullet$ &
	 & 
	$\bullet$ & $\bullet$ & % 
 &
\\
\cline{1-9}
\multicolumn{1}{ |c| }{\cite{LewBonalliEtAl2020}} &
	$\bullet$ & & $\bullet$ &
	 & 
	$\bullet$ & $\bullet$ & % 
 &
\\
\cline{1-9}
\multicolumn{1}{ |c| }{\cite{hewing2018cautious}} &
	$\bullet$ & & $\bullet$ &
	 & 
	$\bullet$ & & % 
 &
\\
\cline{1-9}
\multicolumn{1}{ |c| }{\cite{Drnach2021,Jasour2021}} &
	& $\bullet$ & & $\bullet$
	 & 
	$\bullet$ & $\bullet$ & % 
 $\bullet$ &
\\
\cline{1-9}
\multicolumn{1}{ |c| }{\cite{Han2022}} &
	$\bullet$ & $\bullet$ & $\bullet$ & $\bullet$ 
	 & 
	$\bullet$ & $\bullet$ & % 
 $\bullet$ & 
\\
\cline{1-9}
\multicolumn{1}{ |c| }{\cite{CastilloLopez2020}} &
	$\bullet$ & $\bullet$ & $\bullet$ &  
	 & 
	$\bullet$ & $\bullet$ % 
 & &
\\
\cline{1-9}
\multicolumn{1}{ |c| }{\cite{daSilvaArantes2019,Oldewurtel2008}} &
	$\bullet$ & & $\bullet$ &
	 & 
	& % 
 & &
\\
\cline{1-9}
\multicolumn{1}{ |c| }{\textbf{This work}} & 
	$\bullet$ & $\bullet$ & $\bullet$ & $\bullet$
& 
	$\bullet$ & $\bullet$ & % 
 $\bullet$ & $\bullet$
\\ 
\cline{1-9}
\end{tabular}
\vspace{-4mm}
\end{table*}
\endgroup

\noindent This work indicates that risk-averse planning problems can be efficiently tackled via trajectory optimization. These findings challenge the popular belief that Monte-Carlo-based planning methods are computationally expensive \cite{Thomas2022,FreyRSS2020,LewBonalliEtAl2020}. % 
The key is a particular \avar-constrained formulation which, when coupled with a sample-based approximation and a smooth, sparse reformulation, 
unlocks the use of readily available optimization tools that enable efficient numerical resolution. % 

\section{Related work}\label{sec:related_work}
Risk-averse control, planning, and trajectory optimization  methods typically enforce chance constraints \rev{(CCs)} or average-value-at-risk  (\avar) constraints. % 
\avar constraints are more conservative than \rev{CCs}, as they account for tail events, % 
and have desired properties for robotics applications \cite{ChowPavone2014}. 
Two types of risk constraints  are popular in the literature. % 
First, \textit{pointwise} risk constraints enforce individual constraints % 
at each time $t\in[0,T]$, % 
and thus do not provide sufficient constraints satisfaction guarantees over the \textit{entire} time horizon $[0,T]$,  see Remark \ref{remark:pointwise_cc}. 
Instead, it is preferable to enforce \textit{joint} risk constraints that should hold jointly at all times. 
To enforce such joint risk constraints, many approaches bound the \rev{total} risk of constraints violations using Boole's inequality \cite{MasahiroOno2008,Ma2012,CastilloLopez2020},  
which neglects time correlations of uncertainty. \rev{As a result, these methods yield} solutions that \rev{tend to be conservative and} % 
sensitive to the chosen time discretization   \cite{JansonSchmerlingEtAl2015b,FreyRSS2020,LewEtAl2022}.

The main challenge in risk-averse planning and control lies in the ability of efficiently evaluating 
the risk of constraints violations % 
(e.g., the probability or the \avar of constraints violations over the planning horizon). 
For instance, estimating the joint probability of collision of a robotic system until it completes a task requires  accounting for the effects of disturbances across multiple timesteps, and % 
generally requires evaluating an expectation integral % 
over time and space, which is challenging for general uncertainty distributions and constraints. % 
For this reason, many methods in the literature assume independent Gaussian-distributed random variables defining the problem  \cite{hewing2018cautious,Fan2021,BonalliLewESAIM2022,LewBonalliEtAl2022}, % 
polytopic constraints sets \cite{hewing2018cautious,Hakobyan2019,Ren2023}, 
or use Boole's inequality to distribute risk over different constraints (e.g., obstacles) \cite{hewing2018cautious,LewBonalliEtAl2020,Safaoui2021} which neglects  correlations of uncertainty over the statespace. 
These are reasonable assumptions and approximations in some applications. However, these approaches may lead to infeasibility, under-estimate the true risk of constraints violation, or hinder performance. 
There is a lack of formulations and solution algorithms capable of truly capturing different sources of uncertainty% 
, see Table \ref{table:approaches}. % 

Monte-Carlo methods fulfill the desideratas introduced in Section \ref{sec:introduction}, but are often considered to be computationally expensive \cite{Thomas2022,FreyRSS2020,LewBonalliEtAl2020}. 
Existing Monte-Carlo risk-averse trajectory optimization and planning approaches are limited to problems with randomly-moving obstacles of fixed polytopic shapes for deterministic linear systems \cite{Hakobyan2019}, or require solving multiple problems for different paddings of the obstacles as a function of the constraints violation probability estimated from samples \cite{JansonSchmerlingEtAl2015b}. Both approaches are computationally expensive: the first requires using mixed-integer programming, % 
and the latter requires solving multiple instances of the problem \cite{JansonSchmerlingEtAl2015b}. Methods based on the scenario approach also use samples to enforce chance constraints \cite{Calafiore2006b,Schildbach2014,Bernardini2009,Ahn2022}, but they are limited to convex problems (e.g., to systems with linear dynamics).

Our proposed \rev{continuous-time} risk-constrained problem formulation is more general than in prior work (see Table \ref{table:approaches}), capturing a broad range of trajectory planning problems. % 
In particular, in Section \ref{sec:results}, we study a problem with uncertain nonlinear dynamics, sources of epistemic and aleatoric uncertainty, and moving obstacles, and a legged robot navigating over uncertain terrain whose friction coefficient varies over space. 
Importantly, \rev{the proposed reformulation} is amenable to \rev{efficient numerical resolution}, % 
thanks to \rev{using a particular} \avar-constrained formulation % 
approximated using samples, % 
unlocking the use of \rev{reliable off-the-shelf} optimization tools.% 

\section{Risk functions}\label{sec:risk}
Let $(\Omega,\G,\Prob)$ be a probability space \cite{LeGall2022} and $Z:\Omega\to\R$ be a random variable encoding constraints of the form $Z\leq 0$. For instance, $Z$ may denote the minimum negative distance to obstacles and $Z\leq 0$ may denote obstacle avoidance constraints, see Figure \ref{fig:main_figure_drone}. In practice, enforcing constraints with $\Prob$-probability one (i.e., $Z(\omega)\leq 0$ for $\Prob$-almost all $\omega\in\Omega$) is infeasible, e.g., if disturbances are Gaussian-distributed and have unbounded support. Thus, we may enforce \textit{risk constraints} instead. 
We define the \textit{Value-at-Risk} (\var) and \textit{Average Value-at-Risk} (\avar)\footnote{The \avar is also often referred to as the \textit{Conditional Value-at-Risk} (\cvar) in the literature \cite{MajumdarPavone2017}, since $\avar_\alpha(Z)=\E[Z|Z\geq\var_\alpha(Z)]$ under certain regularity assumptions \cite[Theorem 6.2]{Shapiro2014}.} at risk level $\alpha\in(0,1)$  as
\begin{align}
\label{eq:var}
\var_\alpha(Z)&=
\inf_{t\in\R}\Big(t : \Prob(Z>t)\leq\alpha\Big),
\\
\label{eq:avar}
\avar_\alpha(Z)&=\inf_{t\in\R}\Big(
t+\frac{1}{\alpha}\E[\max(Z-t,0)]
\Big).
\end{align}
$\var_\alpha(Z)$ is the $(1-\alpha)$-quantile of $Z$ ($Z>\var_\alpha(Z)$ with probability less than $\alpha$) and $\avar_\alpha(Z)$ is the expected value of values of $Z$ larger than $\var_\alpha(Z)$ \cite[Theorem 6.2]{Shapiro2014}.  
These risk functions yield two types of inequality constraints. 
First, the $\var_\alpha$ gives chance constraints at probability level $\alpha$:
\begin{equation}
\label{eq:var_prob}
\var_\alpha(Z)\leq 0
\iff
\Prob(Z>0)\leq\alpha.
\end{equation}
The $\avar$ gives conservative formulations of chance constraints since $\var_\alpha(Z)\leq\avar_\alpha(Z)$ \cite{Shapiro2014}:
\begin{equation}\label{eq:avar_prob}
\avar_\alpha(Z)\leq 0
\implies
\Prob(Z>0)\leq\alpha.
\end{equation}

\section{Problem formulation}\label{sec:problem_formulation}
We model the uncertain  system in continuous-time using a stochastic differential equation (SDE). 
Let $(\Omega,\G,\F,\Prob)$ be a filtered probability space \cite{LeGall2022}, 
$W$ be a $n$-dimensional Brownian motion on $\Omega$ \cite{LeGall2016} (e.g., the standard Wiener process), $\xi:\Omega\to\R^q$ be a random variable representing $q\in\N$ uncertain parameters,  
$n,m\in\N$ be state and control dimensions, 
$x_0:\Omega\to\R^n$ be uncertain initial conditions, 
$U\subset\R^m$ be a compact control constraint set, $T>0$ be the planning horizon, 
$b:\R^n\times U\times\R^q\to\R^n$ and $\sigma:\R^n\times U\times\R^q\to\R^{n\times n}$ 
be uncertain drift and diffusion coefficients. Given a control trajectory $u:[0,T]\to U$, we define the SDE 
\begin{align*}
\dd x(t)
&=b(x(t),u(t),\xi)\dt+\sigma(x(t),u(t),\xi)\dd W_t, 
\ t\in[0,T],
\nonumber
\\
x(0)&=x_0,
\tag{\textbf{SDE}}
\label{eq:SDE}
\end{align*}
with solution $x_u$. 
Given running and final costs $\ell:\R^n\times U\to \R$ and 
$\varphi:\R^n\to\R$, 
$N\in\N$  inequality and $n_h\in\N$  equality constraints functions 
$G_j:\R^n\times\R^q\to\R$ and $H:\R^n\to\R^{n_h}$, % 
a risk parameter $\alpha\in(0,1)$, 
and a control space % 
$\U$ (see Sections \ref{sec:theory}-\ref{sec:numerical}),  
we define the optimal control problem (\ocp):
\begin{subequations}
\begin{align}
\ocp: 
\inf_{u\in\U}
\ \  
&\E\left[
\int_0^T\ell(x_u(t),u(t))\dd t 
+
\varphi(x_u(T))
\right]
\label{eq:ocp:cost}
\\[-1mm]
\ \, \text{s.t.}
\quad
&\avar_\alpha\bigg(\sup_{t\in[0,T]}
G(x_u(t),\xi)
\bigg)\leq 0,
\label{eq:ocp:avar}
\\
&\E[H(x_u(T))]=0,
\label{eq:ocp:end}
\\
&x_u\text{ satisfies }\sde,
\label{eq:ocp:sde}
\end{align}
\end{subequations}
where $G(x_u(t),\xi)=\max_{j=1,\dots,N} G_j(x_u(t),\xi)$ is the maximum constraints violation at time $t$ and  we implicitly omit dependencies % 
on uncertainties $\omega$. % 

This formulation of \ocp captures a broad range of robotics applications, see Section \ref{sec:results} for examples.  
\ocp is challenging to solve due to the non-convexity of $G$, the \avar risk constraint \eqref{eq:ocp:avar}, the non-Gaussianity of the state trajectory $x_u$, and the dependency of all quantities on both epistemic uncertainty (modeled by the uncertain parameters $\xi$) and aleatoric uncertainty (modeled with the SDE in \eqref{eq:ocp:sde}). % 

By \eqref{eq:avar}, the \avar constraint \eqref{eq:ocp:avar} in \ocp is equivalent to
\begin{align}
t+\frac{1}{\alpha}\E\bigg[
\max\bigg(
\sup_{s\in[0,T]}
G(x_u(s),\xi)-t,
0\bigg)
\bigg]\leq 0,
\end{align}
so all constraints in \ocp are % 
expected value constraints. Thus, solving \ocp amounts to evaluating expectations, which is generally challenging as it involves a nonlinear SDE and general nonlinear constraints functions. We propose a tractable approximation in \rev{Section \ref{sec:saa}. Next}, we discuss important considerations motivating this formulation.
\begin{remark}[\textbf{joint chance constraints (CCs)}]
\label{remark:ccs_vs_avar} 
Solving \ocp yields feasible solutions to problems with joint CCs. 
Indeed, given a constraint $G(x_u(t),\xi)\,{\leq}\,0$ that should hold at all times   jointly with high probability $1-\alpha$, we can define the joint CC
\begin{equation}\label{eq:joint_cc}
\Prob\Big(
\sup_{t\in[0,T]}
G(x_u(t),\xi)>0
\Big)\leq\alpha.
\end{equation}
Thanks to \eqref{eq:var_prob} and \eqref{eq:avar_prob}, a conservative formulation of the joint CC \eqref{eq:joint_cc} is the corresponding constraint \eqref{eq:ocp:avar} in \ocp.
Thus, % 
by appropriately defining $G$ and solving \ocp,  constraints that often appear in robotics can be enforced with high probability over the \textbf{entire} state trajectory, see Section \ref{sec:results}.
\end{remark}
\begin{remark}[\textbf{pointwise-in-time risk constraints}]\label{remark:pointwise_cc}
Enforcing
\begin{equation}
\avar_\alpha\left(
G(x_u(t),\xi)
\right)\leq 0\quad\forall t\in[0,T],
\end{equation}
instead (e.g., as in \cite{Hakobyan2019}) does not bound the risk of constraints violation over the entire trajectory. 
Similarly, pointwise-in-time chance constraints (e.g., as in \cite{hewing2018cautious,LewBonalliEtAl2020}) do not bound the probability of constraints violations over the  entire planning horizon. Further, transposing discrete-time risk-averse control strategies to full-horizon settings \rev{via Boole's inequality \cite{MasahiroOno2008,Ma2012,CastilloLopez2020} may lead to infeasibility} as the resolution of the time discretization increases, see  \cite{JansonSchmerlingEtAl2015b,FreyRSS2020,LewEtAl2022} for further discussion. 
Placing the time-wise supremum inside the \avar constraint \eqref{eq:ocp:avar} and accounting for time correlations is key to obtaining \rev{trajectory-wise} constraints satisfaction \rev{guarantees}.
\end{remark}

\begin{example}[\textbf{obstacle avoidance constraints}]
\label{example:obs_avoid} 
$N\in\mathbb{N}$  obstacles $\Obs_j$ of uncertain positions and shapes to be avoided at all times can be captured via signed distance functions (SDFs)
$d_{\Obs_j}:\R^n\times\R^q\to\R$, $(x,\xi)\mapsto d_{\Obs_j(\xi)}(x)$ 
such that  $d_{\Obs_j}(x_u(t))\geq 0$ if and only if the system is collision-free, 
i.e., $x_u(t)\notin\Obs_j$. % 
For example, spherical obstacles at $o_j$ of radii $r_j$ can be described by \rev{$d_{\Obs_j}(x)=\|x-o_j\|-r_j$}. 
The fact that $d_{\Obs_j}$ depends on randomness $\omega\in\Omega$ via the uncertain parameters $\xi(\omega)$ allows capturing obstacles of uncertain position and shape. % 
For example, ellipsoidal obstacles centered at $o_j$ of shape matrices $Q_j$  can be encoded by the SDFs 
\begin{equation}\label{eq:sdf:ellipsoid}
d_{\Obs_j(\xi(\omega))}(x)=\rev{(x-o_j(\omega))^\top Q_j(\omega) (x-o_j(\omega))-1}.
\end{equation}
Collision avoidance joint CCs % 
can be written as
\begin{align}\label{eq:cc_obs}
\Prob\big(% 
x_u(t)\notin\Obs_j(\xi)
\
\forall t\in[0,T]\text{ } \forall j=1,\dots,N
\big)
\geq 1-\alpha.
\end{align}
By defining the $N$ constraints function $G_j=-d_{\Obs_j}$, so that % 
\begin{equation}\label{eq:example1:G_obs_max}
G(x_u(t),\xi)
=
\max_{j=1,\dots,N}\left(-d_{\Obs_j(\xi)}(x_u(t))
\right),
\end{equation}
the obstacle avoidance joint CC \eqref{eq:cc_obs} is equivalent to \eqref{eq:joint_cc}. 
Thus, by Remark \ref{remark:ccs_vs_avar},  a conservative reformulation of \eqref{eq:cc_obs} is \eqref{eq:ocp:avar}:
\begin{align}\label{eq:avar:obs_avoid}
\avar_\alpha\bigg(
\sup_{t\in[0,T]}\max_{j=1,\dots,N}\left(-d_{\Obs_j(\xi)}(x_u(t))
\right)
\bigg)
\leq
0.
\end{align}
\end{example}

\section{Sample average approximation (SAA)}\label{sec:saa}
Given $M$ independent and identically distributed (iid) samples $\omega^i\in\Omega$, captured by the multi-sample $\bomega=(\omega^1,\omega^2,\dots)$ with joint distribution $\bProb$, we approximate \ocp by the following sampled optimal control problem (\socp): 
\begin{subequations}
\begin{align}
&\qquad\qquad\qquad
\socp_M(\bomega)
\nonumber
\\
\inf_{\substack{u\in\U\\t\in\R}}
\ \  
&\frac{1}{M}\sum_{i=1}^M
\int_0^T\ell(x_u^i(t),u(t))\dd t 
+
\varphi(x_u^i(T))
\label{eq:socp:cost}
\\
\ \, \text{s.t.}
\quad
&t+
\frac{1}{\alpha M}\sum_{i=1}^M
\max\Big(
\sup_{s\in[0,T]}
G(x_u^i(s),\xi^i)-t,
0\Big)
\leq 0,
\label{eq:socp:avar}
\\[-4mm]
&
-\delta_M\leq
\frac{1}{M}\sum_{i=1}^M
H(x_u^i(T))
\leq\delta_M,
\label{eq:socp:end}
\\
&x_u\text{ satisfies }\sde, % 
\label{eq:socp:sde}
\end{align}
\end{subequations}
where $\left(x_u^i,\xi^i\right)=\left(x_u(\omega^i),\xi(\omega^i)\right)$ for $i=1,\dots,M$,  
\eqref{eq:socp:end} corresponds to  $n_h$ (pointwise) inequality constraints, and 
$\delta_M>0$ is a padding constant that decreases as the number of samples $M$ increases (in practice, we set $\delta_M$ to a small value, see Theorem \ref{thm:optimality}). This approximation is inspired from the sample average approximation (SAA) method \cite{Pagnoncelli2009,Shapiro2014}, which was recently extended to problems with equality constraints that often appear in robotics applications \cite{LewBonalliEtAl2022}. To the best of our knowledge, this approximation has not been applied to continuous-time problems taking the form of \ocp.   

Given $M$ samples $\omega^i\in\Omega$ of the uncertainty, which define $M$ sample paths of the Brownian motion $W(\omega^i)$, of the uncertain parameters $\xi(\omega^i)$, and of the initial conditions $x_0(\omega^i)$, the approximation $\socp_M(\bomega)$ is a tractable deterministic trajectory optimization problem, % 
see Section \ref{sec:numerical}. In the next section, we study the theoretical properties of this approach.

\section{Theoretical analysis}\label{sec:theory}
Depending on the samples $\omega^i$, the computed control trajectories $u_M(\bomega)$ that solve $\socp_M(\bomega)$ will be different.  What can we say about the quality of these solutions? We provide analysis that relies on the following mild assumptions.
\begin{description}% 
\item[\namedlabel{A1}{(A1)}] 
The drift and diffusion coefficients $b$ and $\sigma$ are continuous. Further, there is a bounded constant $K\geq 0$ such that for all $x,y\in\R^n$, all $u,v\in U$, and all values $\xi\in\R^q$, 
$\|b(x,u,\xi)-b(y,v,\xi)\|
+
\|\sigma(x,u,\xi)-\sigma(y,v,\xi)\|
 \leq 
K(\|x-y\|+\|u-v\|)$. 
\item[\namedlabel{A2}{(A2)}]
The cost and constraints functions $(\ell,\varphi,G,H)$ are continuous. Further. there is a bounded constant $L\geq 0$ such that for all $x,y\in\R^n$, all $u,v\in U$, and all values $\xi\in\R^q$,

\vspace{-4mm}

{\small
\begin{align*}
&\|G(x,\xi)-G(y,\xi)\|+\|H(x)-H(y)\|\leq L\|x-y\|,
\\
&\|G(x,\xi)\|+\|H(x)\|\leq L,
\\
&\|\ell(x,u)-\ell(y,v)\|
+
\|\varphi(x)-\varphi(y)\|
\leq
L(\|x{-}y\|+\|u{-}v\|).
\end{align*}
}% 
\item[\namedlabel{A3}{(A3)}] The control space $\U\,{\subset}\,L^2([0,T],U)$ can be identified with a compact subset of $\R^z$ for some $z\in\N$. 
\item[\namedlabel{A4}{(A4)}] 
The uncertain initial state $x_0$ and parameters $\xi$ are $\F_0$-measurable and square-integrable.
\end{description}
\A{1} is standard and guarantees the existence and uniqueness of solutions to \sde. 
\A{2} corresponds to standard smoothness assumptions of the cost and constraints functions. The constraints functions $G$ and $H$ can always be composed with a smooth cut-off function whose support contains the statespace of interest to ensure the satisfaction of the boundedness condition in \A{2}. \A{3} states that the control space $\U$ is finite-dimensional, which is an assumption that is satisfied in practical applications once a numerical resolution scheme is selected. In particular, \A{3} holds for the space of stepwise-constant control inputs $\U=\eqref{eq:control_space}$ that we use in this work. 
\A{4} makes rigorous the interpretation of $x_0$ and  $\xi$ as sources of epistemic uncertainty: \A{4} states that the uncertain initial state $x_0$ and parameters $\xi$ are randomized at the beginning of the episode and are independent of the Brownian motion $W$.

To describe the distance between solutions to $\socp_M$ % 
and % 
to % 
\ocp, given non-empty compact sets $A,B\subseteq\U$, we define
\begin{equation}\label{eq:bD}
\bD(A,B)=\sup_{u\in A}
\Inf_{v\in B}
\|u-v\|.
\end{equation}
As defined above, $\bD$ satisfies the property that $A\subseteq B$ if $\bD(A,B)=0$. Thus, if we show that the solution sets $S_M(\bomega)$ and $S$ of $\socp_M(\bomega)$ of \ocp satisfy  $\bD(S_M(\bomega),S)=0$, then we can conclude that $S_M(\bomega)\subseteq S$, i.e., any optimal solution of $\socp_M(\bomega)$ is an optimal solution of the original problem \ocp. 
Theorem \ref{thm:optimality} below states that this result holds with probability one % 
in the limit as the sample size $M$ increases. 
\begin{theorem}[Asymptotic Optimality]\label{thm:optimality}
Given $M\in\N$ samples and any constants $C>0$ and $\epsilon\in(0,\frac{1}{2})$, define  $\delta_M=CM^{(\epsilon-\frac{1}{2})}$, and 
denote the sets of optimal solutions to $\ocp$ and $\socp_M(\bomega)$ by $S$ and $S_M(\bomega)$, respectively. 
Then, under assumptions \A{1}-\A{4}, $\bProb$-almost-surely, 
$
\lim_{M\to\infty}
\bD(S_M(\bomega),S)=0.
$
\end{theorem}
The proof of Theorem \ref{thm:optimality} follows from recent results in \cite{LewBonalliEtAl2022}, see the appendix. % 
Theorem \ref{thm:optimality} gives convergence guarantees to optimal solutions to \textbf{OCP} (in particular, the \avar constraint \eqref{eq:ocp:avar} is satisfied) as the sample size increases% 
, justifying the proposed approach. 
This asymptotic optimality result \rev{is derived in continuous-time and is thus independent} of the \rev{chosen} discretization scheme for numerical resolution. % 
This contrasts with \rev{methods} that start with a discrete-time formulation and \rev{may} yield different results for different discretizations \cite{JansonSchmerlingEtAl2015b,FreyRSS2020}. % 

The following result gives high-probability \textit{finite-sample} error bounds for the average-value-at-risk constraint given a solution  % 
of the sample-based approximation. % 
The proof (see the appendix) relies on concentration inequalities  \cite{Bartlett2003,Koltchinskii2006,wainwright_2019,LewBonalliEtAl2022}. 

\begin{lemma}[Finite-Sample Error Bound]\label{lem:finite_sample_socp} 
Let $\epsilon>0$, $\beta\in(0,1)$, and $M\in\N$ be such than $M\geq\epsilon^{-2}(\tilde{C}+\bar{h}(2\log(1/\beta))^{\frac{1}{2}})^2$ for some finite constants $(\tilde{C},\bar{h})$ large-enough. 
Denote any solution to $\socp_M(\bomega)$ by $u_M(\bomega)$. 
Then, under assumptions \A{1}-\A{4},  
\begin{equation}
\avar_\alpha\left(\sup_{s\in[0,T]}
G(x_{u_M(\bar\omega)}(s),\xi)
\right)
\leq 
\epsilon
\end{equation}
with $\bProb$-probability at least $(1-\beta)$ over the $M$ iid samples $\omega^i$.
\end{lemma}
By Lemma \ref{lem:finite_sample_socp}, replacing constraint \eqref{eq:socp:avar} in $\socp_M(\bar\omega)$ with 
$$t+\frac{1}{\alpha M}\sum_{i=1}^M
\max\Big(
\sup_{s\in[0,T]}
G(x_u^i(s),\xi^i)-t,
0\Big)
\leq -\epsilon
$$
would guarantee the satisfaction of the average-value-at-risk constraint in \textbf{OCP} if the sample size $M$ is large-enough. \rev{T}he error $\epsilon$ can be made arbitrarily small (with increasingly high probability $1-\beta$) by increasing the sample size $M$. 
In practice, the bound in Lemma \ref{lem:finite_sample_socp} is conservative, and numerical results demonstrate that a few samples are sufficient to obtain high-quality solutions to challenging problems, see Section \ref{sec:results}.  

\section{Numerical resolution}\label{sec:numerical}
Theorem \ref{thm:optimality} justifies approximating \ocp using samples $\omega^i$ and searching for solutions to the deterministic relaxation $\socp_M(\bomega)$ instead. In this section, we describe a numerical method for efficiently computing solutions to  $\socp_M(\bomega)$.

\textbf{Control space parameterization}: 
We optimize over open-loop controls $u$ parameterized by $S\in\N$ stepwise-constant inputs $u_s$ of duration $\Delta t=T/S$, described by the set 
\begin{align}\label{eq:control_space}
\U=\left\{u: 
\begin{array}{l}
u(t)=\sum_{s=0}^{S-1} \rev{u}_s\mathbf{1}_{\left[s\Delta t,(s+1)\Delta t\right)}(t),
\\
(u_0,\dots,u_{S-1})\in U\times\dots\times U
\end{array}
\right\}.
\end{align}
This set $\U$ clearly satisfies \A{3}, since it can be identified with a compact set of $\R^{Sm}$.  Note that any square-integrable open-loop control % 
$u$ % 
can be approximated arbitrarily well by some $u\in\U$ by increasing $S$, so this class of function is expressive.

Alternatively, one could also optimize over certain classes of closed-loop controllers (e.g., controls of the form $u=\bar{u}+Kx$);
an approach that is common in the literature % 
\cite{Oldewurtel2008,Goulart2006}. 

\textbf{Finite-dimensional approximation}: 
Numerically solving general instances of $\socp_M(\bomega)$ requires discretizing the problem. In this work, we discretize $\socp_M(\bomega)$ as follows: % 
$$
\socphat_M(\bomega)
$$
{\small
\\[-12mm]
\begin{subequations}
\begin{align}
\hspace{-2mm}
\inf_{\substack{u\in\U\\t\in\R}}
\  
&\frac{1}{M}\sum_{i=1}^M
\left(
\sum_{k=0}^{S-1}
\ell(x_u^i(k\Delta t),u(k\Delta t))\Delta t+
\varphi(x_u^i(S\Delta t))
\right)
\label{eq:socp_dt:cost}
\\[-1mm]
\hspace{-2mm}\text{s.t.}
\ \,  
&t{+}
\frac{1}{\alpha M}\sum_{i=1}^M
\max\Big(
\max_{k=0,\dots,S}
G(x_u^i(k\Delta t),\xi^i)\,{-}\,t,
0\Big)
\leq 0,
\hspace{-2mm}
\label{eq:socp_dt:avar}
\\[-2mm]
&
-\delta_M\leq
\frac{1}{M}\sum_{i=1}^M
H(x_u^i(S\Delta t))
\leq\delta_M,
\label{eq:socp_dt:end}
\\[-1mm]
&
x_u^i((k+1)\Delta t)=x_u^i(k\Delta t)+b(x_u^i(k\Delta t),u(k\Delta t),\xi^i)\Delta t+
\nonumber
\\[-1mm]
&\hspace{10mm}
+\sigma(x_u^i(k\Delta t),u(k\Delta t),\xi^i)(W_{(k+1)\Delta t}^i-W_{k\Delta t}^i)
\hspace{-1mm}
\label{eq:socp_dt:sde}
\\[-1mm]
&\hspace{18.5mm}
\forall k=0,\dots,S-1, 
\ i=1,\dots,M,
\nonumber
\\[-1mm]
&x_u^i(0)=x_0^i
\ \ \ \forall i=1,\dots,M,
\label{eq:socp_dt:x0}
\end{align}
\end{subequations}
}% 
where $(x_0^i,\xi^i,W^i)=(x_0(\omega^i),\xi(\omega^i),W(\omega^i))$ denote the $M$ samples of the initial conditions, parameters, and sample paths of the Brownian motion, and $x_u^i(k\Delta t)=x_u(k\Delta t, \omega^i)$ denotes associated sample paths of the state trajectory, for conciseness.  
The constraint \eqref{eq:socp_dt:sde} corresponds to an Euler-Maruyama discretization of \eqref{eq:SDE}. 
This discretization of $\socp$ allows for a simple and efficient implementation, although more accurate integration schemes could be used to formulate $\socphat_M(\bomega)$. 

The SDE constraint \eqref{eq:socp_dt:sde} can be either explicitly enforced by optimizing over both state variables $x_u^i(k\Delta t)$ and control variables $u(k\Delta t)$ and enforcing \eqref{eq:socp_dt:sde}, or implicitely by only optimizing over the control variables $u(k\Delta t)$ that parameterize the state trajectory via \eqref{eq:socp_dt:sde}. In this work, we opt for the latter as it reduces the number of variables, albeit at a potential reduction in numerical stability. 
As shown in \cite{DyroHarrisonEtAl2021}, % 
the alternative option of parameterizing both state particles and control variables could also be computationally efficient.

By introducing $M$ additional variables $y_i\in\R$, the inequality constraints  \eqref{eq:socp_dt:avar} are equivalent to the set of the constraints 
\begin{gather}
\begin{cases}
\hspace{3mm}(M\alpha)t+\sum_{i=1}^My_i
\leq 0
\\[1mm]
\hspace{27.6mm}0\leq y_i 
&\forall i=1,\dots,M
\\
G_j(x_u^i(k\Delta t),\xi^i)-t
\leq y_i
&
{\small\begin{cases}
\forall i=1,\dots,M
\\[-1mm]
\forall j=1,\dots,N
\\[-1mm]
\forall k=0,\dots,S.
\end{cases}}% 
\end{cases}
\label{eq:socp_dt:avar_ref}
\end{gather}
Note that these constraints are smooth if every $G_j$ is smooth.

\textbf{Numerical resolution}: With this reformulation, % 
$\socphat$ % 
can be solved using reliable off-the-shelf optimization tools such as % 
\ipopt \cite{ipopt2006}. 
In this work, we leverage  sequential convex programming (SCP). SCP consists of solving a sequence of convex approximations of $\socphat$ until convergence. % 
The main appeal in using SCP is that it typically only requires a few convex approximations of the original non-convex program to reach accurate solutions. 
Since the set of constraints 
\eqref{eq:socp_dt:avar_ref} is sparse, the convex approximations of $\socphat$ can be efficiently solved. % 
In contrast, interior-point-methods for non-convex programming may take a larger number of small steps that each require evaluating gradients and hessians of the original program.  
Since this gradient-hessian evaluation is potentially the computational bottleneck in solving $\socphat$, % 
SCP is a promising solution scheme for this class of problems. 
See  
\cite{LewBonalliEtAl2020,BonalliLewESAIM2022,BonalliLewEtAl2022} for further details on SCP for trajectory optimization.

\begin{remark}[\textbf{\avar vs chance constraints}]\label{remark:avar_easy_optimize_vs_ccs}
Beyond the ability to account for tail events, 
the smoothness of the reformulation in \eqref{eq:socp_dt:avar_ref} motivates enforcing \avar constraints instead of chance constraints that are often used in the literature. % 
Indeed, a typical chance constraint $\Prob(\sup_tG(x_u(t),\xi)>0)\leq\alpha$ is equivalently written as 
$\E[\mathbf{1}_{(0,\infty)}(\sup_tG(x_u(t),\xi))]\leq\alpha$, where $\mathbf{1}_{(0,\infty)}(z)=1$ if $z>0$ and $0$ otherwise. 
This constraint only takes a simple form in particular cases \cite{Calafiore2006}. % 
Approximating this chance constraint from samples would give
$$\frac{1}{M}\sum_{i=1}^M
\mathbf{1}_{(0,\infty)}\Big(\sup_tG(x_u^i(t),\xi^i)\Big)
\leq\alpha,$$
which is not smooth since $\mathbf{1}_{(0,\infty)}(\cdot)$ is a step function. To solve the resulting problem with gradient-based methods, one would need to formulate and solve smooth approximations instead \cite{Geletu2017,PeaOrdieres2020}, or solve the problem multiple times with different constraints paddings  \cite{JansonIchterEtAl2015b}, which is computationally expensive. 
In contrast, the reformulation in \eqref{eq:socp_dt:avar_ref} is smooth and exact (up to errors from the sample-based and discrete-time approximations), which allows efficient numerical resolution. 
\end{remark}

\textbf{Example \ref{example:obs_avoid} (obstacle avoidance constraints).} 
\textit{With $M$ samples of the state trajectory $x_u^i$ and of the obstacles $\Obs_j(\xi^i)$, 
the \avar constraint in \eqref{eq:avar:obs_avoid} can be approximated with 
the set of constraints \eqref{eq:socp_dt:avar_ref} with $G_j(x_u^i(k\Delta t),\xi^i)=-d_{\Obs_j(\xi^i)}(x_u^i(k\Delta t))$. 
This formulation applies to general obstacle representations and can be specialized to particular obstacle shapes. 
For example, if the $N$ uncertain obstacles $\Obs_j$ are spheres of uncertain radii $r_j$ and centers $o_j$, then the last term in \eqref{eq:socp_dt:avar_ref} % 
is % 
\begin{align}
r_j^i-\|x_u^i(k\Delta t)-o_j^i\|-t
\leq y_i
\end{align}
for all  
$i,j,k$, where $(r_j^i,o_j^i)=(r_j,o_j)(\omega^i)$ % 
are $M$ iid samples. 
If the obstacles are ellipsoidal as in \eqref{eq:sdf:ellipsoid}, then \eqref{eq:socp_dt:avar_ref} % 
 becomes 
\begin{equation}
\rev{1-(x_u^i(k\Delta t)-o_j^i)^\top Q_j^i (x_u^i(k\Delta t)-o_j^j)}-t
\leq y_i
\end{equation}
for all  
$i,j,k$.  This gives a set of differentiable constraints that can be passed to a non-convex optimization algorithm. 
}

\textit{We note that potential corner-cutting due to the time discretization of the \avar constraint is easily addressed via different methods, e.g., 
with the approach in \cite{daSilvaArantes2019}, 
by enforcing \eqref{eq:socp_dt:avar} on a finer grid, or by padding obstacles. }

\section{Applications and results}\label{sec:results}
We apply the proposed approach to three challenging planning problems with diverse sources of uncertainty. 
Code is available at \scalebox{0.965}{\url{https://github.com/StanfordASL/RiskAverseTrajOpt}}.

\textbf{1) Drone planning with uncertain obstacles}. % 
The state is  $x=(p,\dot{p})\in\R^6$, 
the input is % 
$u\in\R^3$. Dynamics are given by% 
\begin{equation}
\small
b(x,u,\omega)=
\AverageSmallMatrix{
\dot{p}\\0
}
{+}
\frac{1}{m(\omega)}
\Big(
\rev{-}\beta_\text{drag}\AverageSmallMatrix{
0
\\
|\dot{p}|\dot{p}
}
{+}
\AverageSmallMatrix{
0_{3\times3}\\I_{3\times3}
}
(u+Kx)
\Big),
\end{equation}
and $\sigma(x,u,\omega)=
\frac{1}{m(\omega)}
(0_{3{\times}3},\beta_\sigma I_{3{\times}3})^\top$. 
The drone transports an uncertain payload, modeled by assuming that the total mass $m$ of the system follows a uniform distribution. 
$(\beta_\text{drag},\beta_\sigma)$ % 
are drag and diffusion coefficients and $K$ % 
is a feedback gain.  
We consider three uncertain ellipsoidal obstacles whose shape matrices % 
have uncertain axes distributed according to a uniform distribution, % 
and enforce collision avoidance constraints as described in Example \ref{example:obs_avoid}. 
The objective is reaching a goal ($H(x(T))=x(T)-x_\text{g}$) while minimizing control effort $\ell(x,u)=u^\top Ru$ and 
satisfying collision avoidance constraints at risk level $\alpha$ as in \eqref{eq:avar:obs_avoid}. 
This nonlinear system has sources of aleatoric (disturbances modeled as a Wiener process) and epistemic (mass and obstacles) uncertainty, which makes solving the problem with classical approaches challenging.

We formulate the sample-based approximation with $S=20$ nodes and solve it with SCP.    
We present an example of results in Figure \ref{fig:main_figure_drone} ($M=50$, $\alpha=10\%$) with the trajectory samples $x_u^i$ obtained at convergence. We report the associated Monte-Carlo histogram of the negative minimum distance over $t\in[0,T]$ in \eqref{eq:example1:G_obs_max} with the mean, $\var_\alpha$ and $\avar_\alpha$ minimum values of \eqref{eq:example1:G_obs_max}. The average-value-at-risk of collision is below $0$, indicating that the solution of \socp is feasible for \ocp.

\rev{\textit{Sensitivity to risk parameter:}} We run the approach \rev{($\textrm{SAA}_\alpha$)} % 
for different values of % 
$\alpha$ each % 
and report % 
\rev{results} in Table \ref{table:montecarlo:drone_drving}. 
The $\avar_\alpha$ constraint is \rev{approximately} satisfied  ($\avar_\alpha\,\rev{\approx 0}$) for each value of $\alpha$. As a result, \rev{the} associated joint chance constraints \rev{for collision avoidance} are satisfied, since the percentage of constraints violations is always below $\alpha$, validating the discussion in Remark \ref{remark:ccs_vs_avar}. 
Safer behavior is obtained for smaller values of $\alpha$, effectively  balancing the tradeoff between the efficiency and the risk of constraints violations. 

\rev{\textit{Comparisons:}} \rev{The proposed approach ($\textrm{SAA}_\alpha$) yields trajectories that closely match desired safety risk levels. In contrast}, a baseline  that \rev{neglects} uncertainty \rev{(Deterministic)} often violates constraints. \rev{We also compare with a method (Gaussian${}_\alpha$) that % 
bounds the total probability of constraints violation below $\alpha$ 
using Boole's inequality. This baseline uses an approximate Gaussian-distributed state representation \citep{LewBonalliEtAl2020} and optimizes over the risk allocation, see the appendix for details. Due to the approximate uncertainty representation and the use of Boole's inequality (thus neglecting time and space correlations of uncertainty), this baseline is overly conservative.} % 

 \begin{wrapfigure}{R}{0.4\linewidth}
 \begin{minipage}{0.99\linewidth}
\includegraphics[width=1\linewidth]{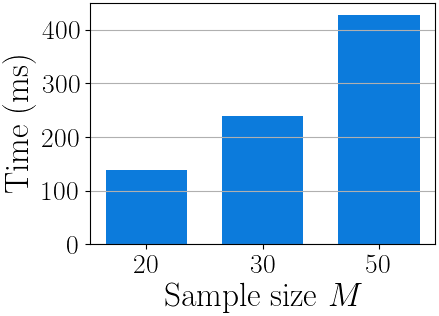}\caption{Drone planning total computation times.}
\label{fig:drone:comp_times_different_M}
  \vspace{-5mm}
\end{minipage}
\end{wrapfigure}

\rev{\textit{Computation time:}} In Figure \ref{fig:drone:comp_times_different_M}, we report total computation times for different sample sizes $M\in\{20,30,50\}$  (we report the median over $30$ runs with $\alpha=5\%$, evaluated on a laptop with an i7-10710U CPU (1.10 GHz) and 16 GB of RAM). We use a zero initial guess $\bar{u}_s=0$ and stop after $10$ SCP iterations, which is sufficient to obtain a final SCP iteration error $\|u^k-u^{k-1}\|/\|u^k\|\leq 1\%$. Albeit our implementation is written in Python (with \textrm{JAX} \cite{jax2018github} and \osqp \cite{Stellato2020}), computation times are reasonable and amenable to real-time applications. 
Computation time scales roughly linearly in the sample size. Parallelization on a GPU could enable using a larger sample size $M$ while retaining speed, albeit results show that using a small sample size suffices to obtain feasible solutions.

\textbf{2) Autonomous driving with a pedestrian}.  
The state is 
$x=(x_\text{ego}, x_\text{ped})$ with $x_\text{ego}=(p^x_\text{ego},p^y_\text{ego},v_\text{ego},\phi_\text{ego})$ the ego-vehicle and $x_\text{ped}=(p^x_\text{ped},p^y_\text{ped},v^x_\text{ped},v^y_\text{ped})$ a pedestrian, and $u=(a,\tau)$ the ego control input. The coupled system follows % 
$\dd x_\text{ego}=(v_\text{ego}\cos(\phi_\text{ego}),v_\text{ego}\sin(\phi_\text{ego}),u)\dt$, 
$\dd p_\text{ped}
=
v_\text{ped}\dt$, and $
\dd v_\text{ped}
=
f(x,\omega)
\dt
+
\sigma
\dd W_t$,  
from $(x_{\text{ego},0},x_{\text{ped},0})$ with known $x_{\text{ego},0}$ but Gaussian-distributed $x_{\text{ped},0}$  due to uncertainty from perception. The term $f(x,\omega)$ represents interaction forces inspired from the Social Force Model \cite{Helbing1995}
\begin{equation}
\small
f(x(t),\omega)=
\omega_1(v_\text{ped}^\text{des}-v_\text{ped}(t))e_{\text{ped}}
+
\omega_2\frac{p_\text{ego}(t)-p_\text{ped}(t)}{\|p_\text{ego}(t)-p_\text{ped}(t)\|},
\end{equation}
where $(\omega_1,\omega_2)$ represents the tendency of a pedestrian to maintain a desired speed $v_\text{ped}^\text{des}$ in the direction $e_{\text{ped}}$ and actively avoid the car. A pedestrian with large values of $\omega_1$ and small values of $\omega_2$ tends to keep the same speed while neglecting the car. $(\omega_1,\omega_2)$ represents a source of epistemic uncertainty: the personality of the pedestrian is randomized and does not change during the planning episode. Such interactions could be learned from data \cite{SalzmannIvanovicEtAl2020,IvanovicElhafsiEtAl2020} and incorporated into our method. 

The objective is reaching a destination $p_\text{ego}^\text{des}$ ($H(x(T))=p_\text{ego}(T)-p_\text{ego}^\text{des}$) while minimizing control effort $\ell(x,u)=u^\top Ru$ and 
maintaining a minimum separation distance $d_\text{sep}$% 
\begin{equation}
\small
\avar_\alpha\Big(\sup_{t\in[0,T]}
d_\text{sep}-\|p_\text{ego}(t)-p_\text{ped}(t)\|
\Big)\leq 0.
\label{eq:driving:separation}
\end{equation} 
This constraint is reformulated as described in Example \ref{example:obs_avoid}. The resulting \ocp problem combines sources of aleatoric (from the Wiener process disturbances) and epistemic uncertainty (from the pedestrian's uncertain behavior and initial pose).

We discretize the problem with $S=20$ nodes % 
and report results in Figure \ref{fig:results:driving} and Table \ref{table:montecarlo:drone_drving}. 
Reducing the risk parameter $\alpha$ yields safer trajectories that maintain a larger distance with the pedestrian at the expense of greater control effort.  For all values of $\alpha$, the \avar constraints and corresponding joint chance constraints are satisfied. In contrast, \rev{the joint-chance-constrained baseline with optimal risk allocation (Gaussian${}_\alpha$) is overly conservative, whereas }a baseline that \rev{neglects} uncertainty \rev{(Deterministic)} is unable to maintain a safe separation distance with the pedestrian at all times. 
Computation times are close to those for the drone planning problem in Figure \ref{fig:drone:comp_times_different_M} (see the appendix), again demonstrating the real-time capabilities of the \rev{proposed} approach.

\begin{figure}[!t]
\centering
\includegraphics[width=1\linewidth]{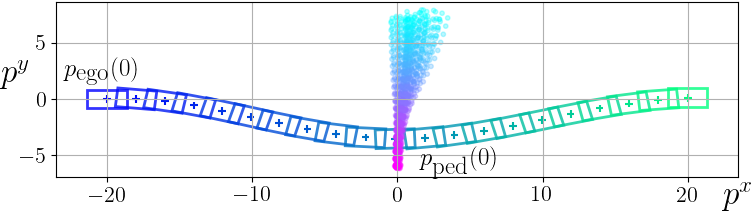}
\caption{Driving scenario results ($\alpha=5\%$).}
\label{fig:results:driving}
\vspace{-3mm}
\end{figure}

\begingroup
\def\arraystretch{1.1}
\setlength\tabcolsep{1.4mm}
\begin{table}[!t]
\caption{Monte-Carlo validation with $\rev{10^4}$ samples of disturbances and uncertain parameters. % 
 We report the median (over $30$ runs with $M=50$) of constraints violation percentage,  risk, and  control effort.}
\label{table:montecarlo:drone_drving}
\centering\begin{tabular}{cc|c|c|c|c|}
\cline{1-6}
\multicolumn{2}{ |c|  }{\textbf{Drone planning}}
& $\alpha=5\%$ & $\alpha=10\%$ & $\alpha=20\%$ & $\alpha=30\%$
\\ 
\cline{1-6}
\multicolumn{1}{ |c  }{\multirow{3}{*}{$\textrm{SAA}_\alpha$} } &
\multicolumn{1}{ |c| }{\begin{tabular}{c}
constraints\\[-1mm] violations
\end{tabular}} & $\rev{2.1}\%$ & $\rev{5.6}\%$ & $\rev{8.4}\%$ & $\rev{11.7}\%$
\\
\cline{2-6}
\multicolumn{1}{ |c  }{}                        &
\multicolumn{1}{ |c| }{{\hspace{-1mm}$\avar_\alpha$\hspace{-4mm}}} & 
$\rev{0.01}$ & $\rev{0.04}$ & $\rev{0.02}$ & $\rev{0.00}$
\\
\cline{2-6}
\multicolumn{1}{ |c  }{}                        &
\multicolumn{1}{ |c| }{cost} & $\rev{76.6}$ & $\rev{54.0}$ & $\rev{48.3}$ & $\rev{46.0}$
\\
\cline{1-6}
\multicolumn{1}{ |c  }{\multirow{2}{*}{\rev{$\textrm{Gaussian}_\alpha$}}} &
\multicolumn{1}{ |c| }{\begin{tabular}{c}
\rev{constraints}\\[-1mm] \rev{violations}
\end{tabular}} & \rev{$1.3\%$} & \rev{$2.3\%$} & \rev{$3.9\%$} & \rev{$6.0\%$}
\\
\cline{2-6}
\multicolumn{1}{ |c  }{}                        &
\multicolumn{1}{ |c| }{\rev{cost}} & \rev{$69.9$} & \rev{$61.8$} & \rev{$54.5$} & \rev{$50.7$}
\\
\cline{1-6}
\multicolumn{1}{ |c  }{\multirow{2}{*}{\hspace{-1mm}\rev{Deterministic}\hspace{-1mm}} } &
\multicolumn{1}{ |c| }{
\begin{tabular}{c}constraints\\[-1mm] violations
\end{tabular}} & \multicolumn{4}{ c| }{$73.8\%$}	
\\
\cline{2-6}
\multicolumn{1}{ |c  }{}                        &
\multicolumn{1}{ |c| }{cost} & \multicolumn{4}{ c| }{$31.4$}
\\
\cline{1-6}
\\[-1mm]
\cline{1-6}
\multicolumn{2}{ |c|  }{\textbf{Autonomous driving}} & $\alpha=1\%$ & $\alpha=2\%$ & $\alpha=5\%$ & $\alpha=10\%$
\\ 
\cline{1-6}
\multicolumn{1}{ |c  }{\multirow{3}{*}{$\textrm{SAA}_\alpha$} } &
\multicolumn{1}{ |c| }{
\begin{tabular}{c}constraints\\[-1mm] violations
\end{tabular}
} & $0.2\%$ & $0.3\%$ & $0.8\%$ & $3.1\%$
\\
\cline{2-6}
\multicolumn{1}{ |c  }{}                        &
\multicolumn{1}{ |c| }{$\avar_\alpha$} & 
$-0.03$ & $-0.05$ & $-0.05$ & $-0.02$
\\
\cline{2-6}
\multicolumn{1}{ |c  }{}                        &
\multicolumn{1}{ |c| }{cost} & $56.3$ & $56.1$ & $53.5$ & $50.9$ 
\\
\cline{1-6}
\multicolumn{1}{ |c  }{\multirow{2}{*}{\rev{$\textrm{Gaussian}_\alpha$} } } &
\multicolumn{1}{ |c| }{\begin{tabular}{c}
\rev{constraints}\\[-1mm] \rev{violations}
\end{tabular}} & \rev{$0.0\%$} & \rev{$0.0\%$} & \rev{$0.3\%$} & \rev{$1.2\%$}
\\
\cline{2-6}
\multicolumn{1}{ |c  }{}                        &
\multicolumn{1}{ |c| }{\rev{cost}} & \rev{$60.2$} & \rev{$58.1$} & \rev{$55.0$} & \rev{$52.5$}
\\
\cline{1-6}
\multicolumn{1}{ |c  }{\multirow{2}{*}{\hspace{-1mm}\rev{Deterministic}\hspace{-1mm}} } &
\multicolumn{1}{ |c| }{\begin{tabular}{c}constraints\\[-1mm] violations
\end{tabular}} & \multicolumn{4}{ c| }{\rev{$53.1\%$}}	
\\
\cline{2-6}
\multicolumn{1}{ |c  }{}                        &
\multicolumn{1}{ |c| }{cost} & \multicolumn{4}{ c| }{\rev{$42.5$}}
\\
\cline{1-6}
\end{tabular}
\vspace{-5mm}
\end{table}
\endgroup

\textbf{3) Legged navigation over uncertain terrain}. Finally, we consider a legged robot whose goal is jumping as far as possible over uncertain terrain while limiting its risk of slippage. 
We consider the hopper robot in \cite{LeCleac2023} with $x=(q,\dot{q})$ where $q=(p_x, p_z, \theta, r)\in\R^4$, $u=(\tau,f)\in\R^2$, and dynamics% 
\begin{equation}\label{eq:dynamics:hopper}
M\ddot{q}(t)+C=B(q(t))u(t)+J_\text{c}(q(t))^\top\lambda(t),
\
x(0)=x^0,
\end{equation}
where $\lambda=(\lambda_x,\lambda_z)\in\R^2$ are contact forces that must satisfy 
\begin{subequations}
\label{eq:hopper}
\begin{gather}
\lambda_x(t)=
\lambda_z(t)=0 \ 
\forall t\in\mathcal{T}_{\text{flight}},
\label{eq:hopper:flight}
\\
J_{\text{c},x}(q(t))^\top\dot{q}(t)
=0, 
\ \ 
\lambda_z(t)\geq 0 \ 
\forall t\in\mathcal{T}_{\text{contact}}, 
\label{eq:hopper:no_slip_qdot}
\\
\hspace{-1mm}\text{and }\  
\avar_\alpha\left(\sup_{t\in\mathcal{T}_{\text{contact}}}
\lambda_x(t) - \mu(q(t),\omega)\lambda_z(t)
\right)\leq 0.
\label{eq:hopper:noslip}
\end{gather}
\end{subequations}
These constraints encode the absence of contact force in the flight phase  \eqref{eq:hopper:flight}, the no-slip and positive normal force conditions \eqref{eq:hopper:no_slip_qdot}, and a constraint on the risk of slippage \eqref{eq:hopper:noslip}, over a pre-defined contact schedule $\mathcal{T}_{\text{flight}},\mathcal{T}_{\text{contact}}\subset[0,T]$. 

Characterizing terrain adhesion properties (i.e., terramechanics) is challenging and 
is often done via data-driven approaches. For instance,  \citet{Cunningham2017} fit Gaussian processes to observed data for slip prediction. Thus, we model the soil properties at the contact point $p_{\text{foot},x}(q)$ given by the system's kinematics with the Random Fourier Features \cite{RahimiRecht2007}
$
\mu(q,\omega)=\bar\mu + \sum_{n=1}^{30} \omega_{n,1}\cos(\omega_{n,2}\cdot  p_{\text{foot},x}(q)+\omega_{n,3})
$ for some 
 \begin{wrapfigure}{R}{0.38\linewidth}
 \begin{minipage}{0.99\linewidth}
   \vspace{-4mm}
 \centering
 \includegraphics[width=1\linewidth]{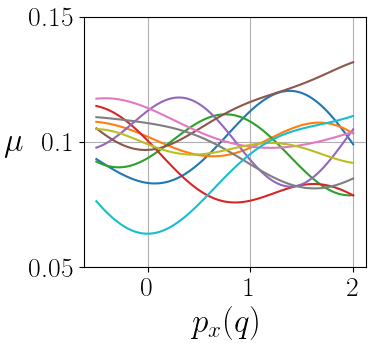}
 \caption{Samples from $\mu$.}
 \label{fig:friction}
 \vspace{-5mm}
 \end{minipage}% 
 \end{wrapfigure}
 
\noindent randomized parameters $\omega$ following a uniform distribution, see Figure \ref{fig:friction}. 
The coefficient $\mu$ encodes epistemic uncertainty varying over space, which is more realistic than assuming a constant friction coefficient as in \cite{Drnach2021}. As discussed in Section \ref{sec:related_work}, this type of uncertainty is challenging to deal with.

The objective consists of jumping as far as possible (i.e., $\varphi(x(T))=-p_x(T)$) while minimizing the control effort $\ell(x,u)=\|u\|_2^2$ and limiting the risk of slippage \eqref{eq:hopper:noslip}. % 
The resulting problem clearly takes the form of \ocp with no dynamics uncertainty (i.e., $\sigma=0$). Indeed, an important source of uncertainty in legged locomotion comes from uncertain terrain properties due to imperfect perception or inherent uncertainty in the problem. 
We are not aware of prior work in trajectory optimization tackling this formulation. % 

We solve the sampled problem\rev{s} (\rev{using} $S=M=30$) with \ipopt \cite{ipopt2006} \rev{and} % 
report results in Figure  \ref{fig:hopper:jumps}. 
By reducing the risk parameter $\alpha$, the \rev{proposed} trajectory optimizer is more cautious to avoid slip and returns shorter jumps,  
with slippage probability bounded by $\alpha$. % 
In contrast, \rev{a deterministic baseline that only considers the mean terrain parameter $\mu(x,\omega)=\bar\mu$ returns longer jumps but} violates constraints $50\%$ of the time.

\begin{figure}[!t]
\centering
\includegraphics[width=1\linewidth]{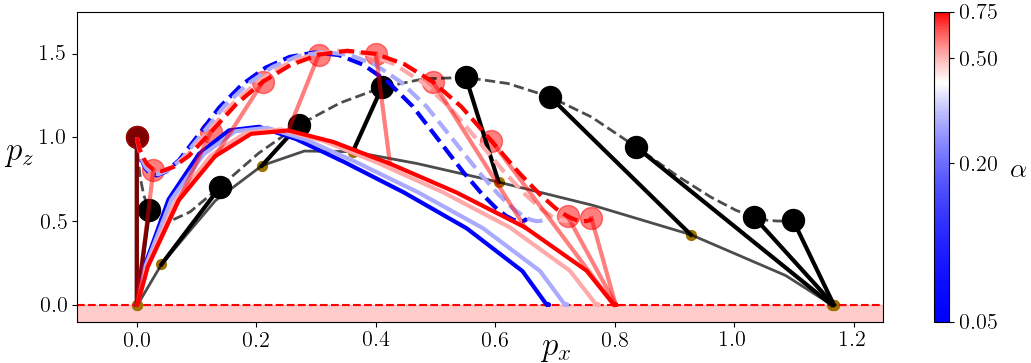}
\caption{Hopper system with uncertain terrain properties: jump trajectories (of the center of mass and the tip of the leg) for different risk parameter \rev{values} $\alpha$ and for the \rev{deterministic} baseline (in black).}
\label{fig:hopper:jumps}
\vspace{-6mm}
\end{figure}

\section{Conclusion} 
\label{sec:conclusion}
We proposed an efficient method for risk-averse trajectory optimization. 
This algorithm hinges on a continuous-time formulation with average-value-at-risk constraints. 
By approximating the problem using samples, % 
we obtain a smooth, sparse program that allows for efficient numerical resolution.  
We demonstrated the speed and reliability of the method on problems with sources of epistemic and aleatoric uncertainty that are challenging to tackle with existing approaches.

Due to its generality in handling uncertain nonlinear dynamics and constraints, this work 
opens exciting future research directions. 
In particular, the approach could be interfaced with learned models: 
obstacles could be implicitly represented as deep SDFs \cite{Park2019} or neural radiance fields \cite{mildenhall2020nerf,Adamkiewicz2022}, % 
deep trajectory forecasting models \cite{SalzmannIvanovicEtAl2020,IvanovicElhafsiEtAl2020}
could be used for planning in autonomous driving, 
learned terramechanics models \cite{Cunningham2017} could allow more robust legged locomotion,  
and dynamics could learned online, % 
either to improve control performance \cite{deisenroth2015,hewing2018cautious}  or for active data-gathering  \cite{Zhang2021,LewEtAl2022} % 
under risk constraints. % 

\makeatletter
\newcommand*\smaller{% 
  \@setfontsize\smaller{7}{8.}% 
}
\makeatother

\let\oldbibliography\thebibliography
\renewcommand{\thebibliography}[1]{\oldbibliography{#1}
\setlength{\itemsep}{0pt}} % 
\bibliographystyle{plainnat}
{
\smaller
\bibliography{ASL_papers,main}
}

\appendix

\subsection{Proofs of Theorem \ref{thm:optimality} and of Lemma \ref{lem:finite_sample_socp}}\label{appendix:proofs}

\begin{proof}[Proof of Theorem \ref{thm:optimality}]
Thanks to \A{1}-\A{4}, the solution to \sde is well-defined \cite{LeGall2016,BonalliBonnet2022} and all functions of interest are measurable  (see \cite[Appendix A.2]{LewBonalliEtAl2022}. In particular, the map $\omega\in\Omega\mapsto t+\alpha^{-1}\max(\sup_{s\in[0,T]}G(x_u(s,\omega),\xi(\omega))-t,0)\in\R$ is measurable for all $(u,t)\in\U\times\R$ since $z\mapsto\max(z, 0)$ is continuous). 
Further, using Kolmogorov's Lemma \cite[Theorem 2.9]{LeGall2016},  one can show that up to a modification, the map from controls $u\in\U$ to state trajectories $x_u$ is H\"older continuous (for the uniform norm on the space of continuous functions $C([0,T],\R^n)$) with probability one, see \cite[Appendix E]{LewBonalliEtAl2022}.

Thanks to \A{2} and since the $\max$ is $1$-Lipschitz continuous, the functions $G$ and $H$ are Lipschitz continuous. Thus, the maps $(u,t)\mapsto t+\alpha^{-1}\max(\sup_{s\in[0,T]}G(x_u(s),\xi)-t,0)$ and $u\mapsto H(x_u(T))$ are H\"older continuous with probability one. 

Thanks to \A{2}, % 
the optimal risk parameter $t^\star$ is achieved in a bounded interval $\mathcal{T}\subset\R$, see % 
\cite{Shapiro2014}. Thus, thanks to \A{3}, \ocp is a stochastic program over a compact subset of $\R^{z+1}$. 

The conclusion follows from \cite[Theorem 5.1]{LewBonalliEtAl2022}. % 
\end{proof}

\section{Expectations concentration inequalities}
The proof of Lemma \ref{lem:finite_sample_socp} relies on a concentration inequality from \cite{LewBonalliEtAl2022}. Let $(\Omega,\G,\Prob)$ be a probability space, 
$d\in\N$, 
$\U\subset\R^d$ be a compact set, 
$h:\U\times\Omega\to\R$ be Carath\'eodory (i.e., $h(u,\cdot)$ is $\G$-measurable for all $u\in\U$ and $h(\cdot,\omega)$ is continuous $\Prob$-almost-surely), and consider the following two assumptions.
\begin{description}
\item[\namedlabel{B1}{(B1)}] $\Prob$-almost-surely, the map $u\mapsto h(u,\omega)$ is $\tilde\alpha$-H\"older continuous for some exponent $\tilde\alpha\in(0,1]$ and H\"older constant $M(\omega)$ satisfying $\E[M(\cdot)^2]<\infty$, such that 
\begin{equation*}
|h(u_1,\omega)-h(u_2,\omega)|\leq M(\omega)\|u_1-u_2\|_2^{\tilde\alpha}\ \ 
\forall u_1,u_2\in\U.
\end{equation*}
\item[\namedlabel{B2}{(B2)}] For some $\bar{h}<\infty$, $\Prob$-almost-surely, $\sup\limits_{u\in\U} |h(u,\omega)|\leq \bar{h}$. 
\end{description}
Let $\{\omega^i\}_{i=1}^M$ be $M\in\N$ independent and identically distributed (iid) samples of $\omega$.  
We have the following result.

\begin{lemma}[Concentration Inequality \cite{LewBonalliEtAl2022}]\label{lem:concent:unif_bounded}
Assume that $h$ satisfies \Bassum{1} and \Bassum{2}. Let $D=2\sup_{u\in\U}\|u\|_2$, 
$C=256$,  
$u_0\in\U$, 
$\Sigma_0$ denote the covariance matrix of $h(u_0,\cdot)$, and 
\begin{equation}
\tilde{C}=\left(
CD^{\frac{\tilde\alpha+1}{2}}d^{\frac{1}{2}}\E\left[M^2\right]^{\frac{1}{2}}\tilde\alpha^{-\frac{1}{2}}
+
\text{Trace}\left(\Sigma_0\right)^{\frac{1}{2}}
\right).
\end{equation}
Let $\epsilon>0$, $\beta\in(0,1)$, and the sample size $M\in\N$ be such than $M\geq\epsilon^{-2}(\tilde{C}+\bar{h}(2\log(1/\beta))^{\frac{1}{2}})^2$. 
Then, with $\bProb$-probability at least $1-\beta$ over the $M$ iid samples $\omega^i$, 
\begin{equation}
\sup_{u\in\U}\left|\frac{1}{M}\sum_{i=1}^Mh(u,\omega^i)-\E[h(u,\cdot)]\right|
\leq\epsilon.
\end{equation}
\end{lemma}

Lemma \ref{lem:concent:unif_bounded} provides a high-probability finite-sample error bound for the sample average approximation of the expected value of $h(u,\cdot)$ that holds uniformly over all $u\in\U$. % 

To prove Lemma \ref{lem:finite_sample_socp}, we define the map $Z:\U\times\Omega\to\R$ by 
\begin{equation}\label{eq:Zu}
Z_u(\omega)=\sup_{s\in[0,T]}
G(x_u(s,\omega),\xi(\omega)),
\end{equation}
which is bounded and H\"older continuous under assumptions \A{1}-\A{4}. Then, given a large-enough compact set $\mathcal{T}\subset\R$, we define the measurable map $g:(\mathcal{T}\times\U)\times\Omega\mapsto\R$,
\begin{equation}\label{eq:g}
g((t,u),\omega)\mapsto t+\alpha^{-1}\max(
Z_u(\omega)
-t, 0).
\end{equation}
Since $g$ is a composition of H\"older continuous functions % 
and $\mathcal{T}$ is bounded, $g$ is also bounded and H\"older continuous, i.e., $g$ satisfies assumptions \Bassum{1} and \Bassum{2} over $(t,u)\in\mathcal{T}\times\mathcal{U}$. % 

\begin{corollary}[Finite-Sample Error Bound]\label{cor:finite_sample_g_tu}
Given $M$ iid samples $\omega^i\in\Omega$, formulate the sample average approximation $\socp_M(\bomega)$ of \ocp. 
Define the map $g$ as in \eqref{eq:g}. 
Let $\epsilon>0$ and $\beta\in(0,1)$. 
Then, under assumptions \A{1}-\A{4},   assuming $M\geq\epsilon^{-2}(\tilde{C}+\bar{h}(2\log(1/\beta))^{\frac{1}{2}})^2$ for $(\tilde{C},\bar{h})$ large-enough, 
\begin{equation}
\sup_{(t,u)\in\mathcal{T}\times\U}\left|\frac{1}{M}\sum_{i=1}^Mg((t,u),\omega^i)-\E[g((t,u),\cdot)]\right|
\leq\epsilon
\end{equation}
with $\bProb$-probability at least $(1-\beta)$.
\end{corollary}

\begin{proof}
Under assumptions \A{1}-\A{4}, $g$ satisfies Assumptions \Bassum{1}-\Bassum{2}. The conclusion follows from Lemma \ref{lem:concent:unif_bounded}. 
\end{proof}

 The desired result (Lemma \ref{lem:finite_sample_socp}) follows % 
 from Corollary \ref{cor:finite_sample_g_tu}. % 

\textit{Proof of Lemma \ref{lem:finite_sample_socp}:}
For any $u\in\mathcal{U}$, define the random variable $Z_u(\omega)=\eqref{eq:Zu}$, so that 
$$
\avar_\alpha\left(\sup_{s\in[0,T]}
G(x_{u(\bar\omega)}(s),\xi)
\right)
=
\avar_\alpha\left(Z_{u(\bar\omega)}
\right).
$$
Then, given an interval $\mathcal{T}$ large-enough containing the optimal risk parameter $t^\star$ of \ocp % 
(see the proof of Theorem 1 and \cite{Shapiro2014}), define the function  $g:(\mathcal{T}\times\U)\times\Omega\mapsto\R$ as in \eqref{eq:g}. Denote the solution to $\socp_M(\bar\omega)$ by $(t(\bar\omega),u(\bar\omega))$ and
$$
g((t(\bar\omega),u(\bar\omega)),\omega)
=
t(\bar\omega)+\alpha^{-1}
\max(Z_{u(\bar\omega)}(\omega)-t(\bar\omega),0).
$$
Since $(t(\bomega),u(\bomega))$ solves  $\socp_M(\bar\omega)$, 
\begin{align*}
\frac{1}{M}\sum_{i=1}^Mg((t(\bar\omega),u(\bar\omega)),\omega^i)
&
\\[-4mm]
&\hspace{-3cm}=
t(\bar\omega)+\frac{1}{\alpha M}\sum_{i=1}^M\max(Z_{u(\bar\omega)}(\omega^i)-t(\bar\omega),0) \leq 0.
\end{align*}
By Corollary \eqref{cor:finite_sample_g_tu}, with probability at least $1-\beta$, % 
\begin{equation*}
\sup_{(t,u)\in\mathcal{T}\times\U}\left|\frac{1}{M}\sum_{i=1}^Mg((t,u),\omega^i)-\E[g((t,u),\cdot)]\right|
\leq\epsilon.
\end{equation*}
From the last two inequalities\footnote{Assuming $t(\bomega)\in\mathcal{T}$, which can be enforced by restricting the search of solutions to $\socp_M(\bomega)$ to the compact set $\U\times\mathcal{T}$ with $\mathcal{T}$ arbitrarily large.}, with probability at least $1-\beta$,
$$
\E[g((t(\bar\omega),u(\bar\omega)),\cdot)]
\leq\epsilon.
$$
We obtain that
\begin{align*}
\avar_\alpha(Z_{u(\bar\omega)})
&=
\inf_{t\in\R}\Big(
t+\alpha^{-1}\E[\max(Z_{u(\bar\omega)}-t,0)]
\Big)
\\
&\leq
t(\bar\omega)+\alpha^{-1}\E[\max(Z_{u(\bar\omega)}-t(\bar\omega),0)].
\\
&=
\E[t(\bar\omega)+\alpha^{-1}\max(Z_{u(\bar\omega)}-t(\bar\omega),0)].
\\
&=
\E[g((t(\bar\omega),u(\bar\omega)),\cdot)]
\\
&\leq\epsilon
\end{align*}
with probability at least $1-\beta$, concluding the proof. 
\hfill $\blacksquare$

\subsection{Smooth reformulation \eqref{eq:socp_dt:avar_ref} of \avar constraint in $\socphat$}
$\socphat$ is a deterministic non-convex program. Due to the $\max$ operation, its inequality constraint \eqref{eq:socp_dt:avar} from the \avar constraint is not smooth even if each term in $G$ is smooth Many optimization tools use gradient information and assume that constraints are twice differentiable, which is not the case with \eqref{eq:socp_dt:avar}. 
This constraint can be reformulated by introducing $M$ variables $y_i\in\R$ and reformulating  \eqref{eq:socp_dt:avar} as follows:

\vspace{-3mm}

{\small
\begin{gather}
\quad \, 
t+
\frac{1}{\alpha M}\sum_{i=1}^M
\max\left(
\max_{k=0,\dots,S}
G(x_u^i(k\Delta t),\xi^i)-t,
0\right)
\leq 0
\nonumber
\\
\iff
\nonumber
\\
\begin{cases}
(M\alpha)t+\sum_{i=1}^My_i
\leq 0
\nonumber
\\
y_i \geq 0
&\forall i=1,\dots,M
\\
y_i \geq
\max_{k=0,\dots,S}
G(x_u^i(k\Delta t),\xi^i)-t
&\forall i=1,\dots,M
\end{cases}
\nonumber
\\
\iff
\nonumber
\\
\begin{cases}
(M\alpha)t+\sum_{i=1}^My_i
\leq 0
\nonumber
\\
y_i \geq 0
&\forall i=1,\dots,M
\\
y_i \geq
\max_{k=0,\dots,S}^{j=1,\dots,N}
G_j(x_u^i(k\Delta t),\xi^i)-t
&\forall i=1,\dots,M
\end{cases}
\nonumber
\\
\iff
\nonumber
\\
\qquad 
\begin{cases}
\hspace{3.2mm}
(M\alpha)t+\sum_{i=1}^My_i
\leq 0
\\[1mm]
\hspace{26mm}0\leq y_i 
&\forall i=1,\dots,M
\\
G_j(x_u^i(k\Delta t),\xi^i)-t
\leq y_i
&
{\small\begin{cases}
\forall i=1,\dots,M
\\
\forall j=1,\dots,N
\\
\forall k=0,\dots,S.
\end{cases}}% 
\end{cases}
\qquad \eqref{eq:socp_dt:avar_ref}
\nonumber
\end{gather}
}% 
In contrast to \eqref{eq:socp_dt:avar}, the set of constraints  \eqref{eq:socp_dt:avar_ref} is smooth if every $G_j$ is smooth. With \eqref{eq:socp_dt:avar_ref} and the additional variable $y=(y_1,\dots,y_M)$, we reformulate $\socphat$ as follows:
$$
\socphat_M(\bomega)
$$
$$
\inf_{\substack{u\in\U, t\in\R, y\in\R^M}}
\ \eqref{eq:socp_dt:cost}(u)\  
\text{ such that }\ 
\eqref{eq:socp_dt:end}, \,  \eqref{eq:socp_dt:sde}, \, \eqref{eq:socp_dt:x0}, \, \eqref{eq:socp_dt:avar_ref}.
$$
$\socphat$ % 
can be solved using off-the-shelf optimization tools. % 

\subsection{Additional implementation details and results}
We open-source code to reproduce the experiments in Section \ref{sec:results} at \scalebox{0.965}{\url{https://github.com/StanfordASL/RiskAverseTrajOpt}}.  
Our implementation is in Python and uses \textrm{JAX} \cite{jax2018github} to evaluate constraints gradients. 
To solve $\socphat$, for the drone planning and autonomous driving problems, we use a standard sequential convex programming (SCP) scheme. We use \osqp \cite{Stellato2020} to solve the convexified problems at each SCP iteration.  For the hopper system, we solve $\socphat$ using \ipopt \cite{ipopt2006} due to its improved numerical stability compared to SCP.

\subsubsection{Drone planning problem}
The computation time generally depends on the choice of solver, desired accuracy, and use case. For example, using our proposed approach within a risk-aware MPC controller would allow warm-starting for faster convergence, see also \cite{Diehl2005}. 
In this work, we use $\bar{u}_s=0$ as a zero initial guess, disable the obstacle avoidance risk constraint for the first two SCP iterations to obtain more robust convergence, set the error tolerance threshold for \osqp to $10^{-3}$, and use the normalized $L^2$-error over control iterates $\|u^k-u^{k-1}\|/\|u^k\|$ to detect the  convergence of SCP. 
We report computation times and accuracy statistics in Figure \ref{fig:drone:comp_times_SCP}, % 
evaluated on a laptop with an i7-10710U CPU (1.10 GHz) and 16 GB of RAM. We observe that $10$ SCP iterations are sufficient to obtain accurate results. Albeit our implementation is written in Python, computation times are reasonable and amenable to real-time applications. By warm-starting within an MPC loop, one could achieve a replanning frequency of $30\,\textrm{Hz}$ using $M=30$ samples and by returning the solution after a single SCP iteration.   
Finally, computation time scales roughly linearly in the number of samples. Parallelization (e.g., on a GPU) would enable using a larger numbers of samples $M$ while retaining speed, albeit our results show that a small number of samples is sufficient to obtain feasible solutions.

\begin{figure}[!t]
\centering
\includegraphics[width=1\linewidth]{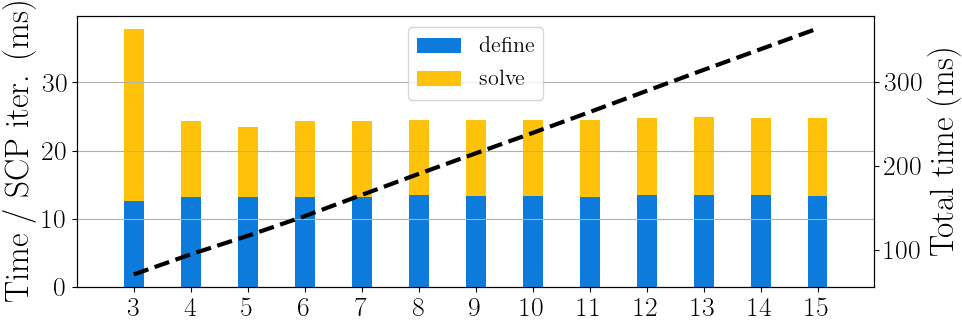}
\vspace{-3mm}
\includegraphics[width=1\linewidth]{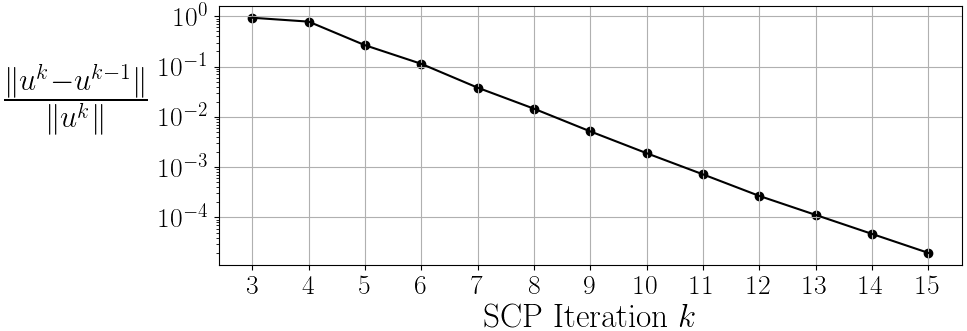}

\vspace{2mm}

\caption{Drone planning problem: median computation time and SCP iteration error statistics over $30$ experiments ($\alpha=0.05$, $M=30$).}
\label{fig:drone:comp_times_SCP}

\vspace{3mm}
\centering
\includegraphics[width=1\linewidth]{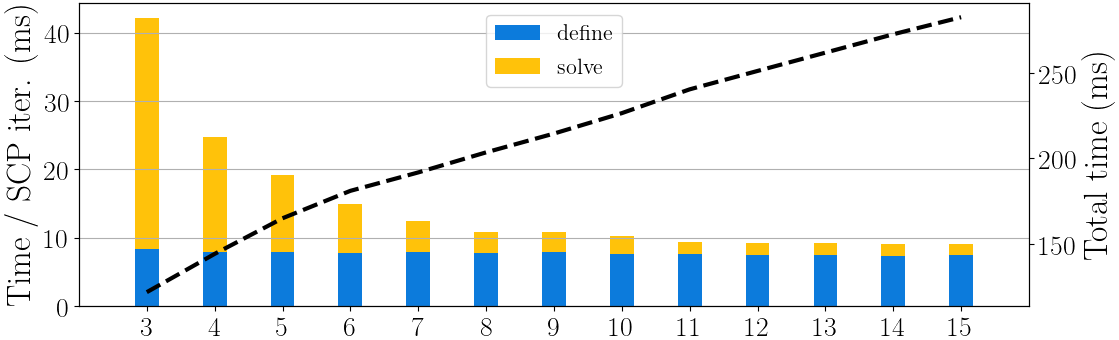}
\vspace{-3mm}
\includegraphics[width=1\linewidth]{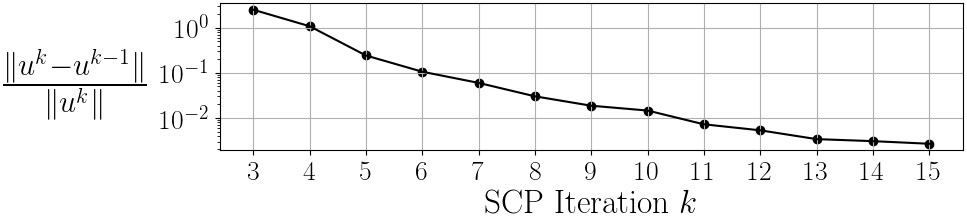}

\vspace{2mm}

\begin{minipage}{0.5\linewidth}
\caption{Autonomous driving planning problem. Median statistics over $30$ experiments ($\alpha=0.02$, $M=50$). 
Top plots: computation times and SCP iteration errors. 
Right plots: computation times for $10$ SCP iterations and different risk parameters $\alpha$.}
\label{fig:driving:comp_times}
\end{minipage}
\hspace{0.02\linewidth}
\begin{minipage}{0.45\linewidth}
\includegraphics[width=1\linewidth]{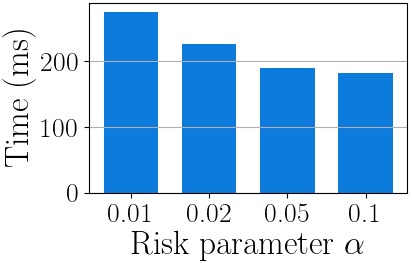}
\end{minipage}

\vspace{-7mm}
\end{figure}

\subsubsection{Autonomous driving problem}
Computation times and SCP statistics are reported in Figure \ref{fig:driving:comp_times}. % 
Results are comparable to those for the drone planning problem, again  demonstrating the real-time capabilities of the approach. Computation times do not appear to be sensitive to the risk parameter $\alpha$.

\subsubsection{Hopper robot problem}
We consider the hopper robot in \cite{LeCleac2023} with $x=(q,\dot{q})$ where $q=(p_x, p_z, \theta, r)\in\R^4$, $u=(\tau,f)\in\R^2$. The system follows deterministic dynamics
\begin{equation*}
M\ddot{q}(t)+C=B(q(t))u(t)+J_\text{c}(q(t))^\top\lambda(t), 
\
x(0)=x^0,
\quad\eqref{eq:dynamics:hopper}
\end{equation*}
where $M\in\R^{4\times 4}$ is the inertia matrix, $C\in\R^4$ comprises of Coriolis and conservative forces, $B(q(t))\in\R^{4\times 2}$ is the control matrix, $J_\textrm{c}(q(t))\in\R^{4\times 2}$ is the contact Jacobian, and 
$\lambda=(\lambda_x,\lambda_z)\in\R^2$ are contact forces that must satisfy \eqref{eq:hopper}. 

Monte-Carlo validation of the satisfaction of the risk constraint for the trajectories represented in Figure \ref{fig:hopper:jumps} (obtained with a sample size $M=30$) are provided in 
Table \ref{table:hopper:montecarlo}.
 
\begingroup
\def\arraystretch{1}
\setlength\tabcolsep{1.6mm}
\begin{table}[!t]
\caption{Hopper system. Monte-Carlo validation for $10^5$ samples of the terrain parameter $\mu(q,\omega)$. We report the percentage of failures (the number of times the no-slip constraint in \eqref{eq:hopper:noslip} is violated), 
 the empirical $\avar_\alpha$ value, 
 and the final jump distance $p_x(T)$.}
\label{table:hopper:montecarlo}
\centering\begin{tabular}{cc|c|c|c|c|}
\cline{3-6}
& & $\alpha=5\%$ & $\alpha=10\%$ & $\alpha=20\%$ & $\alpha=50\%$
\\ 
\cline{1-6}
\multicolumn{1}{ |c  }{\multirow{3}{*}{$\textrm{SAA}_\alpha$} } &
\multicolumn{1}{ |c| }{\begin{tabular}{c}constraints\\ violations
\end{tabular}} & $0.5\%$ & $1.2\%$ & $2.2\%$ & $7.1\%$
\\
\cline{2-6}
\multicolumn{1}{ |c  }{}                        &
\multicolumn{1}{ |c| }{$\avar_\alpha$} & 
$-0.06$ & $-0.08$ & $-0.09$ & $-0.10$
\\
\cline{2-6}
\multicolumn{1}{ |c  }{}                        &
\multicolumn{1}{ |c| }{$p_x(T)$} & $0.58$ & $0.61$ & $0.64$ & $0.69$
\\
\cline{1-6}
\multicolumn{1}{ |c  }{\multirow{2}{*}{Baseline} } &
\multicolumn{1}{ |c| }{\begin{tabular}{c}constraints\\ violations
\end{tabular}} & \multicolumn{4}{ c| }{$50.2\%$}	
\\
\cline{2-6}
\multicolumn{1}{ |c  }{}                        &
\multicolumn{1}{ |c| }{$p_x(T)$} & \multicolumn{4}{ c| }{$1.10$}
\\
\cline{1-6}
\end{tabular}
\end{table}
\endgroup

\subsection{
\rev{Chance-constrained baseline with risk allocation}}
\rev{The risk-averse baseline used in the comparisons tackles the discrete-time joint-chance-constrained problem
\begin{subequations}
\label{eq:discrete_time:joint_cc_problem}
\begin{align}
\label{eq:discrete_time:cost}
\inf_{u\in U^S}
\ \  
&\E\left[
\sum_{k=0}^{S-1}\ell(x_k,u_k) 
+
\varphi(x_S)
\right]
\\
\ \, \text{s.t.}
\quad
&\Prob\bigg(\max_{k=1,\dots,S}G(x_k,\xi)>0
\bigg)\leq \alpha,
\label{eq:discrete_time:joint_cc}
\\
\label{eq:discrete_time:equality}
&\E[H(x_S)]=0,
\\
\label{eq:discrete_time:dynamics}
&
x_{k+1}
=b(x_k,u_k,\xi)\Delta t+\sigma(x_k,u_k,\xi)w_k, 
\end{align}
\end{subequations}
where $G(x_k,\xi)=\max_{j=1,\dots,N}
G_j(x_k,\xi)$ (see Section \ref{sec:problem_formulation}), $\Delta t=T/S$, $U^S=U\times\dots\times U$ ($S$ times), and each $w_k$ is (iid) Gaussian-distributed of mean $0$ and variance $\Delta t$.}

\rev{To conservatively reformulate the joint chance constraint \eqref{eq:discrete_time:joint_cc}, this baseline uses Boole's inequality, which states that
\begin{align}
\label{eq:discrete_time:Booles}
\hspace{-2mm}
\Prob\Big(\max_{k=1,\dots,S}
G(x_k,\xi)>0
\Big)
& 
\leq 
\sum_{k=1}^S
\sum_{j=1}^N
\Prob\left(
G_j(x_k,\xi)>0
\right).
\end{align}
Thus, by introducing $NS$ variables $\alpha_{jk}\in(0,1)$, the optimization problem in \eqref{eq:discrete_time:joint_cc_problem} can be replaced by the more conservative optimization problem with \textit{pointwise} chance constraints
\begin{subequations}
\label{eq:discrete_time:pointwise_cc_problem}
\begin{align}
\inf_{(u_k,\alpha_k)}
\ \  
&\eqref{eq:discrete_time:cost}
\quad
\text{s.t.}
\quad
\eqref{eq:discrete_time:equality},\, \eqref{eq:discrete_time:dynamics},
\\
&\Prob\Big(
G_j(x_k,\xi)>0
\Big)\leq \alpha_{jk}\hspace{-1mm}
&&{\small\left(
\begin{matrix}
\, \forall j=1,\dots,N
\\
\forall k=1,\dots,S
\end{matrix}
\right)}
\label{eq:discrete_time:pointwise_cc}
\\
&\alpha_{jk}>0\ 
&&{\small\left(
\begin{matrix}
\, \forall j=1,\dots,N
\\
\forall k=1,\dots,S
\end{matrix}
\right)}
\\
&\sum_{k=1}^S\sum_{j=1}^N\alpha_{jk}\leq\alpha.
\end{align}
\end{subequations}
To approximately solve this problem \eqref{eq:discrete_time:pointwise_cc_problem}, we use the standard method described next. 
The state trajectory is modeled as Gaussian-distributed ($x_k\sim\mathcal{N}(\mu_k,\Sigma_k)$) with mean and covariance trajectories computed approximately with a linearization approach \citep{LewBonalliEtAl2020}. 
Since the terms $G_j(x_k,\xi)$ correspond to collision avoidance constraints with convex obstacles in the drone and driving problems, we linearize the terms $G_j(x_k,\xi)$ in $x_k$, which enables reformulating the corresponding pointwise chance constraints in \eqref{eq:discrete_time:pointwise_cc} using the $(1-\alpha_k)$-th quantiles of the inverse
cumulative function of the standard normal distribution \citep[Equations (5) and (26)]{LewBonalliEtAl2020}. The resulting program is solved using off-the-shelf optimization techniques.}

\rev{As discussed in Remark \ref{remark:pointwise_cc} and in details in \cite{JansonSchmerlingEtAl2015b,FreyRSS2020}, this baseline is conservative due to the use of Boole's inequality in \eqref{eq:discrete_time:Booles}. Using this inequality, this baseline neglects correlations of uncertainty. As such, the conservatism of this baseline increases with the number $S$ of discretization timesteps \cite{JansonSchmerlingEtAl2015b,FreyRSS2020}. }

\end{document}